\documentclass[a4paper]{llncs}

\usepackage{amsmath}
\usepackage{amsfonts}
\usepackage{eucal}
\usepackage{multicol}
\usepackage{amsmath,amssymb}
\usepackage{lipsum}
\usepackage{balance}

\usepackage{graphics} 
\usepackage{epsfig} 
\usepackage{mathptmx} 
\usepackage{algorithm,algorithmic}

\newcommand{\sysname}{SyROF}
\newcommand{\sysnamelong}{\emph{Synchronous Robotic Framework}}

\newcommand{\byte}{{\bf uint8\_t} ~}
\newcommand{\inte}{{\bf uint32\_t}~}
\newcommand{\void}{{\bf void}~}

\begin{document}
\title{\sysnamelong}


\author{Nagarathna Hema Balaji\ \and Jyothsna Kilaru \and Oscar Morales-Ponce}
\institute{Department of Computer Engineering and Computer Science, California State University Long Beach \\
\email{rathnahb01@gmail.com, Jyothsna.Kilaru@student.csulb.edu, oscar.morales-ponce@csulb.edu}}

\maketitle

\begin{abstract}
We present a synchronous robotic testbed called \sysname~that allows fast implementation of robotic swarms. Our main goal is to lower the entry barriers  to cooperative-robot systems for undergraduate and graduate students. The testbed provides a high-level programming environment that allows the implementation of Timed Input/Output Automata (TIOA). \sysname~offers the following unique characteristics: 1) a transparent mechanism to synchronize robot maneuvers, 2) a membership service with a failure detector, and 3) a transparent service to provide common knowledge in every round. These characteristics are fundamental to simplifying the implementation of robotic swarms. The software is organized in five layers: The lower layer consists of a real-time publish-subscribe system that allows efficient communication between tasks. The next layer is an implementation of a Kalman filter to estimate the position, orientation, and speed of the robot. The third layer consists of a synchronizer that synchronously executes the robot maneuvers, provides common knowledge to all the active participants, and handles failures. The fifth layer consists of the programming environment. 
\end{abstract}

\section{Introduction}

In natural disasters, the response time for the search-and-rescue team is a critical factor to minimize the casualties. A swarm of robots, e.g. drones can be used to concurrently sweep the region to minimize the response time without risk to the search-and-rescue team~\cite{rosenfeld2017intelligent}.  Some researchers have followed a centralized approach where the robots are supervised and operated from a central control~\cite{doroodgar2014learning}, \cite{chien2012scheduling}. However, these solutions are not scalable and are prone to failure. More robust solutions focus on fully autonomous systems, where the robots self-coordinate their actions without the help of preexisting infrastructure \cite{balta2017integrated}. However, the implementation of these systems requires an interdisciplinary team with expertise in robotics, distributed systems, wireless sensor networks, etc, making the systems accessible only for experienced research teams. In this paper we present a testbed, called \sysnamelong~(\sysname), that gives access to undergraduate computer science/engineer at  California State University Long Beach  to rapidly  implement and demonstrate robotic swarms. In essence, \sysname~provides the \emph{look-compute-move} model proposed by Suzuki and Yamashita in their seminar paper~\cite{suzuki1993formation}. In our framework, robots broadcast their state infinitely often. Then, every participant robot obtains the state of every other active robot using a membership service, which can be used to compute its next movement. In the \emph{look-compute-move} model, we can distinguish two main variants: 1) the fully synchronous $\mathcal{FSYNC}$, where all robots start the cycles at the same time \cite{flocchini2008arbitrary}, and  2) the asynchronous $\mathcal{ASYNC}$, where no assumption is made about the cycles. We choose to implement $\mathcal{FSYNC}$ to reduce the complexity of implementing cooperative algorithms. The power of the model has been shown in \cite{cohen2008local}, \cite{yamauchi2013pattern}, \cite{yamauchi2017plane}.  However,  \sysname~ can also be used to implement the asynchronous variant. We implement a  synchronizer to execute the synchronized cycle. In $\mathcal{FSYNC}$ every robot has common knowledge, i.e., every robot knows that every other robot knows, and so on, the state of the participant robots. Thus, the output of any deterministic function is identical in all robots. However, their output can be conflicting if robots do not attain common knowledge. We consider wireless networks where messages can be dropped at any time which makes it impossible to attain common knowledge deterministically. To overcome the impossibility, we consider the stream consensus protocol proposed by Morales-Ponce et al., \cite{morales2018stop} that guarantees common knowledge at almost every time with period of disagreement of bounded length.


In this paper we describe the design of \sysname~testbed that consists of multiple assorted mobile  robots including omnidirectional robots, drones and rovers. Each robot is equipped with a microprocessor, a gyroscope/accelerometer sensor, a flow sensor, a GPS like sensor and a  Bluetooth chip. We  design and implement a real-time publish-subscribe system as a base system. Then we design and implement a Kalman filter to compute the state of the robot that reduces the noise of the sensory data. Then we implement a synchronizer to synchronize the task of the robots. On top of the synchronizer, we implement a programing interface, called virtual machine, that allows implementing Timed Input/Output Automata. The main software architecture is shown in Figure~\ref{fig:arch}. Our design allows adding external sensors such as small computers for more complex applications. 

\begin{figure}[!h]
\centering
\includegraphics[scale=0.7]{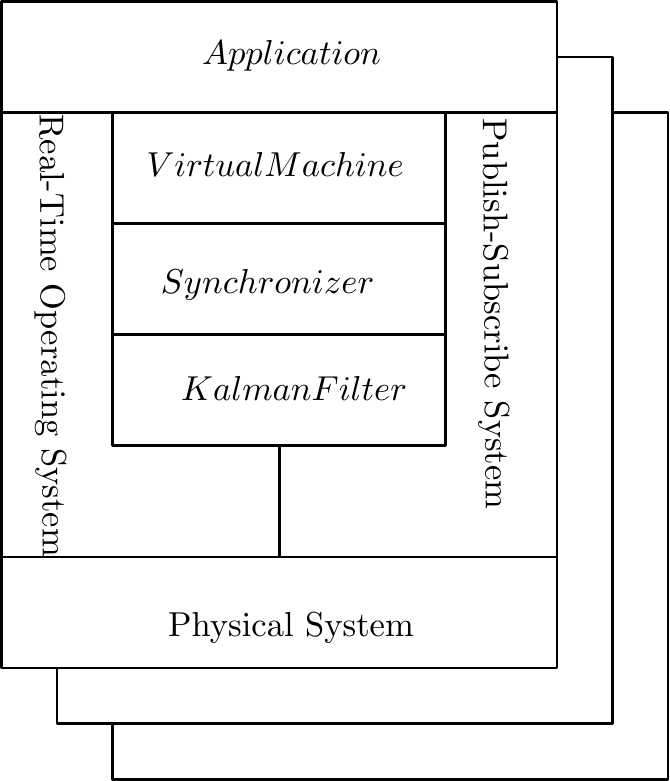}
\caption{Software architecture}
\label{fig:arch}
\end{figure}

The paper is organized as follows: in Section~\ref{sec:model}, we present the architecture of the robots. The formal requirements of the system are then presented in Section~\ref{sec:requirements}, followed by the related work in Section~\ref{sec:relatedwork}. We present the design for the real-time publish-subscribe system in Section~\ref{sec:pubsub}, then we present the design of the Kalman filter in Section~\ref{sec:kalman} and the membership service and synchronizer in Section~ \ref{sec:ync}. The programming interface  is presented in Section~\ref{sec:language}.  Finally, we conclude the paper in  Section~\ref{sec:conclusion}.

\section{Robot Architecture}\label{sec:model}
We design the main components of the system to be compatible with different robotic models. Currently, we have built the testbed with omnidirectional robots that are able to move in any direction without rotations. However, the components and software were designed to support different robot models, such as rovers that can rotate and move forward and backward, non-holonomic mobile robots, drones, etc. 

The main components of  \sysname~are implemented in a Cortex M4 with a real rime operating system. In particular, we use Teensy 3.2 and FreeRTOS for the omnidirectional robots that allows executing tasks  in bounded time.  \sysname~uses different sensors to compute the robot's state. Namely, an MPU-9265 (Inertial Measurement Unit or IMU) to obtain the robot's orientation,  a PMW390 (optical flow sensor) to compute the distance traversed by the robots, an NRF51822 (bluetooth module) to communicate and a DWM1000 (an ultra-wide band sensor) to compute the robot's absolute position. \sysname~also controls all the actuators. 
To simplify the connections, we design a printed circuit board (PCB) depicted in Figure~\ref{fig:pcb}.

\begin{figure}[!h]
\centering
\includegraphics[width=3cm, angle=90]{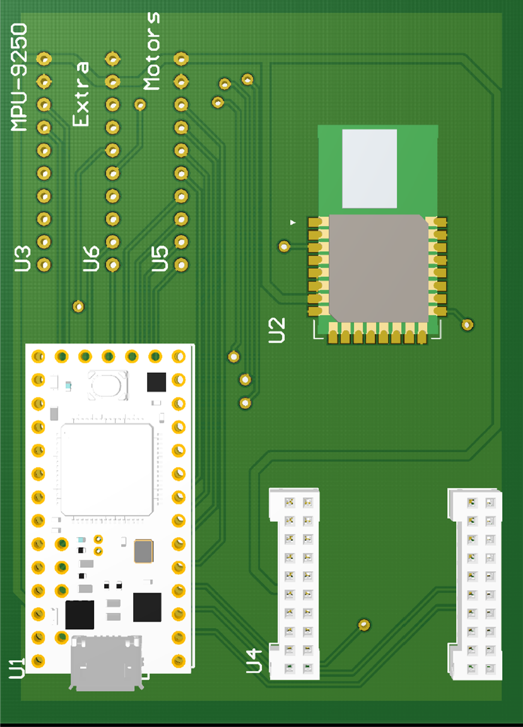}
\caption{Autopilot PCB }
\label{fig:pcb}
\end{figure}

Many robotic applications can be enhanced with the use of computer vision and machine learning algorithms. However,  these algorithms require more computational power. Hence, we design the system to be highly configurable to include different sensors and actuators including external small computers.
Throughout the paper we assume that robots are in communication range and that the clock drift of the process is bounded by a constant. 

\section{Requirements} \label{sec:requirements}
In this section, we list the main functional requirements of \sysname. 

\begin{itemize}
\item \emph{Provide a membership service}. In robotic swarms, every robot needs to be aware of every other robot that is participating in the system. This requirement adds dynamism to the system by allowing robots to join and leave at any time.  

\item \emph{Synchronize robotic primitives}. The time to complete a maneuver (or primitive) is non-deterministic due to inaccuracies in the actuators and sensors, battery levels, etc. Robots should synchronize primitives to provide the look-compute-move $\mathcal{FSYNC}$ model. 

\item \emph{Provide common knowledge infinitely often, even with communication faults}. Common knowledge is attained when every robot knows the state of all participant robots. If  all robots  communicate successfully, they must switch to cooperative mode and provide their state including position, speed, intentions, etc., to every other robot. If any robot does not receive the message of at least one robot, all robots must switch to autonomous mode. 

\item \emph{Provide a high level programming environment to implement cooperative robotic systems using Timed Input/Output Automata (TIOA)}. The programming environment should provide  instructions to control the actuators, read sensory data, provide complex queries. 
\end{itemize}

\section{Related Work}\label{sec:relatedwork}

Testbed for lowering the entrance barrier to multi-robot systems have been proposed previously, see for example  \cite{pickem2015gritsbot},  \cite{rubenstein2012kilobot}, where the the authors present a testbed, called GRITSBot and Kilobots, respectively, using inexpensive robots. 
In GRITSBot,  a central server is responsible for executing the code in each robot. Moreover, it requires a precise tracking system to determine the position
of each robot. Thus, user can design and test robotic synchronous system. However, any failure in the server affects the whole system. Unlike GRITSBot, in \sysname~all the robots have the same role and, therefore, there is no single point of failure. 
In the Kilobot testbed robots asynchronously  runs its own code. However, those robots do not have enough computational power to estimate their positions.  Unlike Kilobots, our proposed testbed, robots can estimate their own position and synchronize the primitives.

Another closely related framework is ROS (Robot Operating System)~\cite{quigleyros}.  A system built using ROS consists of several processes connected in a peer-to-peer topology.  ROS has been very successful and many packages for a great variety of applications are available including SLAM (Simultaneously Localization and Mapping). However, ROS does not provide support to synchronize robotic primitives nor common knowledge guarantees. Our framework provides a new high-level platform-independent programming language that provides users with synchronous  executions of robotic primitives and common knowledge in realtime. Moreover, \sysname~ can be integrated with ROS to extend its functionalities without compromising the realtime execution.
 
Our synchronization mechanism is based on consensus among all robots. 
Many papers have based their systems on consensus  where processes need to agree on specific tasks~\cite{davis2016consensus, wei2015new, varadharajan2018over}. 
Another approach is to use of tuple spaces. Processes insert tuples in a common space that can be read by every process~\cite{pinciroli2016tuple}, but, there is no guarantee that a process is able to read.

\section{Real-Time Publish-Subscribe System}\label{sec:pubsub}
In this section we present the design of the real-time publish-subscribe system that simplifies the implementation of robotic systems. The publish-subscribe system can be logically separated into three parts- the Broker, Publishers, and Subscribers as shown in Figure~\ref{fig:pubsub}. Each process can be either publisher or subscriber or both. 

\begin{itemize}
\item \emph{The Broker.}
The main functionality of the broker is to disseminate data from the publisher to subscribers that are interested in the topic. Thus, it plays a key role for the expected functionality of the system.  Publishers register through the Broker using unique identifiers and subscribers subscribe to the publisher with the Broker using the identifiers. It is assumed that identifiers are unique.  

\begin{figure}[!h]
\centering
\includegraphics[width=8cm]{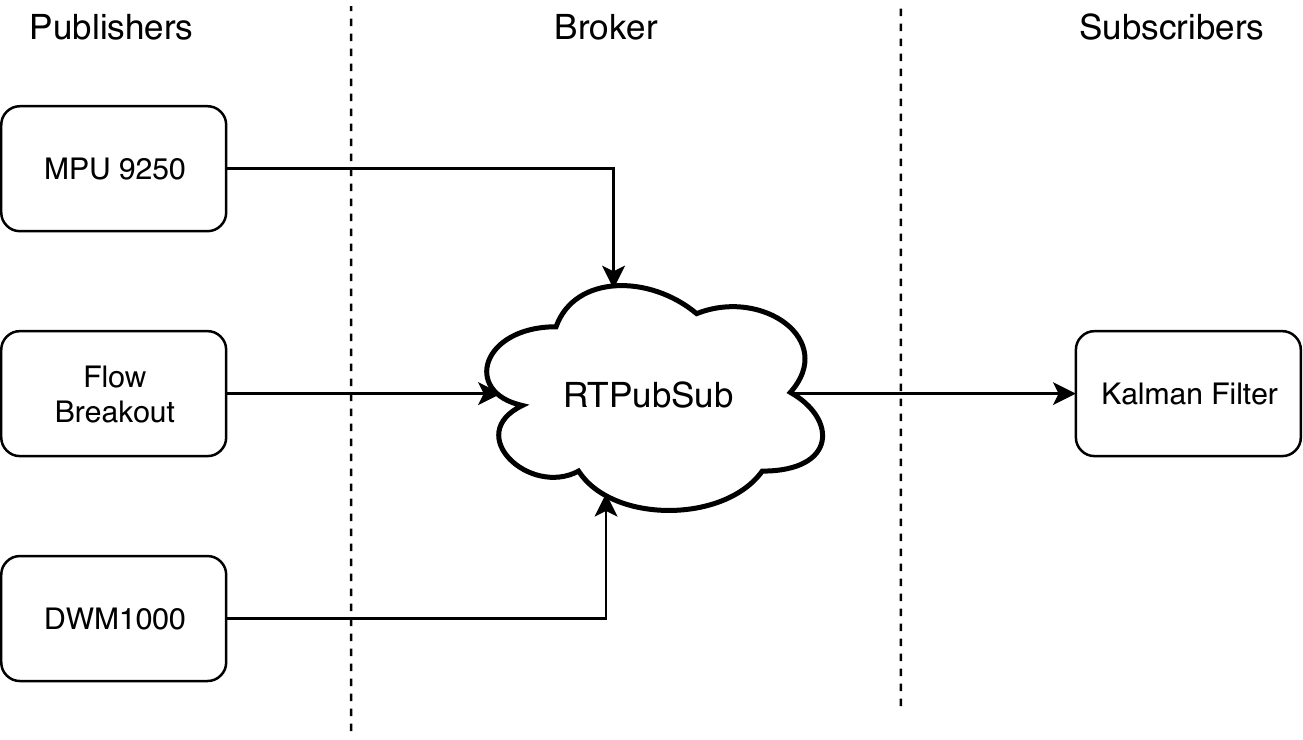}
\caption{Publish Subscribe System}
\label{fig:pubsub}
\end{figure}




\item \emph{Publishers.} They are dedicated to produce data without knowing  which processes consume it.  For example, publishers can read data from sensors, execute an algorithm to filter data, etc. When a publisher is created, it registers as publisher  using a unique identifier.  When data ready, it notifies the broker. 
 


\item \emph{Subscribers.} They are dedicated to consume data. When a subscriber is created, it subscribes  to  the topic that it is interested in using their unique identifier. It can subscribe to more than one topic. When a publisher publishes the data, the broker notifies all the subscribers that have declared interest in it. 
Subscribers can also play the role of publisher. 

\end{itemize}


In this paper we report the implementation in FreeRTOS.  However, the design can be also implemented in other real-time operating systems. FreeRTOS is based on co-routines to emulate a multiprocessor system. Essentially, a co-routine can have multiple entry points and maintains its state between activations. In FreeRTOS, each process executes a task that does not  return during the execution of the program.  To switch to other tasks, an operating system method must be invoked such as \emph{vTaskDelayUntil}, \emph{vTaskDelay}, etc. Since we are interested in providing a system for time-critical applications, each task must complete a step of computation in bounded time.

\subsection{Broker Interface}



The real-time publish-subscribe system interface is presented in Interface~\ref{alg:signature} which consists of three methods
\begin{itemize}
\item registerPublisher.  Register a publisher into the real-time publish-subscribe system using a unique id. First, the publisher task instantiates a FreeRTOS queue using the operating system method $xQueueCreate$ and passes the handler to the system as well as the size of the tuple. The Broker inserts these values into a list of publishers. The function returns the handler (index) of the publisher which is  later used to publish new data.





\item subscribe. Register a subscriber to a given publisher. Before the subscriber calls $subscribe$, it creates a task and a queue that stores the publisher produced tuple. In FreeRTOS, $xCreateTask$  and $xQueueCreate$  are used to create the task and the queue, respectively. These values are passed as parameters to the function $subscribe$ and the Broker inserts these two values into a publisher queue.

\item publish. Publishers use the  function  $publish$ to publish new data. To avoid blocking, the broker notifies the main broker task. When the broker receives the notification, it reads the data from the publisher using the queue handler and then copy it to all the subscribers using their queue handlers.

\end{itemize}



\floatname{algorithm}{Interface}
\begin{algorithm}[!h]
\begin{algorithmic}[1]
\caption{Publish-Subscribe Interface}
\label{alg:signature}
\STATE \byte registerPublisher(\byte id,  \byte size, {\bf QueueHandle\_t} queueHandler)%
\STATE  {\bf bool} subscribe(\byte id, \inte taskHandler, {\bf QueueHandle\_t} queueHandler) 
\STATE  \byte publish(\byte index) 

\end{algorithmic}
\end{algorithm}

\section{Kalman Filter}\label{sec:kalman}
In this section we present the implementation of the Kalman filter~\cite{julier1997new}. The Kalman filter estimates the state of the system in real-time from noisy sensory information, assuming that the noise follows a Gaussian distribution. The Kalman filter is critical for the correct functionality of our testbed.
Sensors suffer from errors occurring due to many reasons. For instance, GPS sensors might suffer deviation while the device does not have a direct sight-of-view to the satellite. Furthermore, odometric values must be accurate to determine the actual position of the vehicle. 

\subsection{Kalman Filter Design}
Let $\Delta$ be the rate at which the Kalman filter runs and $X(t) = (x(t), y(t), \dot{x}(t), \dot{y}(t), \theta(t))$ be the robot's state at time $t$ where $x(t)$ and $y(t)$ are the absolute position, $\dot{x}(t)$ and $\dot{y}(t)$ are the speed and 
$\theta(t)$ the orientation of the robot. Since the current state only depends on the previous  state,  we omit $t$. Let $P$ be the covariance matrix of $X$ at time $t$. $P$ represents the uncertainties in the current state. Let $u_a$ be the motor thrust  and $u_\theta$ be the steering angle. Based on the dynamic model, we can compute the new state as follows:

\[
\begin{array}{ll}
\hat{x} = x + \Delta \frac{(\dot{x})^2}{2} \\
\hat{y} = y + \Delta \frac{(\dot{y})^2}{2} \\
\hat{\dot{x}} = \dot{x} + \Delta  \frac{\mu_a^2}{2} \cos(\Delta  \frac{\mu_\theta^2}{2}) \\
\hat{\dot{y}} = \dot{y} + \Delta  \frac{\mu_a^2}{2} \sin(\Delta  \frac{\mu_\theta^2}{2}) \\
\hat{\theta} = \theta + \Delta \frac{\mu_\theta^2}{2}  \\
\end{array} 
\]

It can be simple written as  $\hat{X} = f(X) + f(U)$.
Let $J$ be the Jacobian matrix of $f(X)$ and $Q$ be the covariance matrix representing 
the process noise. The estimation of covariance matrix at time $t+1$ is defined as  
$\hat{P} = F P F^T + Q$.

Let $Z = (gyro_z, flow_x, pos_x, pos_y)$ be the sensor readings at time $t$. 
Let $R$ be the covariance matrix measuring the noise. 
We use the following equations to deduce $h(\hat{X})$. 
\[
\begin{array}{lcl}
flow_x &=& \hat{\dot{x}} \\
flow_y &=& \hat{\dot{y}} \\
pos_x &=&  \hat{x} \\
pos_y &=& \hat{y} \\ 
gyro_z &=& \hat{\dot{\theta}} \\
\end{array} 
\]

Let $F$ be  the system state matrix. We can summarize the Kalman filter
as:

Model forecast
\begin{equation} \label{eqn:model}
\begin{array}{lcl}
\hat{X} &=& f(X) + f(U)  \\
\hat{P}   &=& F P F^T + Q  
\end{array} 
\end{equation}

Step correction
\begin{equation} \label{eqn:correction}
\begin{array}{lcl}
K &=& P F^T (F P F^T + R)^{-1} \\
X &=& \hat{X}^T + K(Z - h(\hat{X}))  \\
P &=& (I - K J_h) \hat{P}
\end{array} 
\end{equation}

\floatname{algorithm}{TIOA}
\begin{algorithm}
\begin{multicols}{2}
\begin{algorithmic}[1]
\caption{Synchronization protocol}
\label{alg:actions2}
\STATE { Signature: }
\STATE \hspace{0.1cm} {\bf ext }  $SendP  \langle${Id, mbrInP, OM, $Progress$} $\rangle$ 
\STATE \hspace{0.1cm} {\bf ext }  $SendW\langle$ {Id, mbrInW, OM, $Wait$} $\rangle$ 
\STATE \hspace{0.1cm} {\bf ext }  $SendV\langle$ {Id, mbrInV, OM, $Vote$} $\rangle$ 
\STATE \hspace{0.1cm} {\bf ext }  $Recieve\langle$ {Id, information,OM, state} $\rangle$

\STATE \hspace{0.1cm} {\bf ext }  $FailureDetector$ 
\STATE \hspace{0.1cm} {\bf ext }  $GoToAutonomous$
\STATE \hspace{0.1cm} {\bf ext }  $GoToCooperative$

\STATE {\bf State: }
\STATE \hspace{0.1cm} {\bf int } $Progress$
\STATE \hspace{0.1cm} {\bf int } $Wait$
\STATE \hspace{0.1cm} {\bf int } $Vote$

\STATE \hspace{0.1cm} {\bf analog} $now \in \mathcal{R}$ {\bf initially} $0$
\STATE \hspace{0.1cm} {\bf analog} $nextSend \in \mathcal{R}$ {\bf initially} $0$

\STATE \hspace{0.1cm} {\bf discrete} $infromation$, is the the set of member which can be list of nodes in Progress, Wait or Vote
\STATE \hspace{0.1cm} {\bf discrete} $OM$, Operation Mode {\bf initially} $\emptyset$
 
\STATE \hspace{0.1cm} {\bf discrete} $state \in \{Progress, Waite,  Vote \}$, state of the nodes  
 
\STATE \hspace{0.1cm} {\bf discrete} $memberInP$, is a list of the member in Progress state
\STATE \hspace{0.1cm} {\bf discrete} $memberInW$, is a list of the member in Wait state
\STATE \hspace{0.1cm} {\bf discrete} $memberInV$, is a list of the member in Vote state

\STATE \hspace{0.1cm} {\bf discrete} $messageLost$, is a list to keep track of message lost for each process

\STATE {\bf Actions: }
\STATE {\bf external} $SendP\langle${Id, mbrInP, OM, $Progress$} $\rangle$ 
\STATE \hspace{0.1cm} {\bf precondition}:
\STATE \hspace{0.2cm} $state = Progress$ $\wedge$ $nextSend \gets now + \alpha$
\STATE \hspace{0.1cm} {\bf effect}:
\STATE \hspace{0.2cm} $nextSend \gets now + \alpha$

\STATE {\bf external}  $SendW \langle$ {Id, mbrInW, OM, $Wait$} $\rangle$ 
\STATE \hspace{0.1cm} {\bf precondition}:
\STATE \hspace{0.2cm} $state = Wait$$\wedge$ $nextSend \gets now + \alpha$
\STATE \hspace{0.1cm} {\bf effect}:
\STATE \hspace{0.2cm} $nextSend \gets now + \alpha$

\STATE {\bf external}  $SendV \langle$ {Id, mbrInV, OM, $Vote$} $\rangle$ 
\STATE \hspace{0.1cm} {\bf precondition}:
\STATE \hspace{0.2cm} $state = Vote$$\wedge$ $nextSend \gets now + \alpha$
\STATE \hspace{0.1cm} {\bf effect}:
\STATE \hspace{0.2cm} $nextSend \gets now + \alpha$

\STATE {\bf external} $receive\langle$ {Id, information,OM, state} $\rangle$ 
\STATE \hspace{0.1cm} {\bf precondition}:
\STATE \hspace{0.2cm} $full(Chanel) = True$
\STATE \hspace{0.1cm} {\bf effect}:

\STATE \hspace{0.2cm} {\bf if} $state = PROGRESS $
\STATE \label{ln:collected} \hspace{0.3cm} $ mbrInP \gets mbrInP   \cup \{id\}$ 

\STATE \hspace{0.2cm} {\bf if} $state = Wait $
\STATE \label{ln:collected} \hspace{0.3cm} $ mbrInW \gets mbrInW   \cup \{id\}$  

\STATE \hspace{0.2cm} {\bf if} $state = Vote $
\STATE \label{ln:collected} \hspace{0.3cm} $ mbrInV \gets mbrInV  \cup \{id\}$ 

\STATE {\bf external} $FailureDetector$ 
\STATE \hspace{0.1cm} {\bf precondition}:
\STATE \hspace{0.2cm} $full(Chanel) = False$
\STATE \hspace{0.1cm} {\bf effect}:

\STATE \label{ln:collected} \hspace{0.2cm} $ messageLost_Id \gets  messageLost_Id   +  1$ 
 
\STATE {\bf external} $GoToAutonomous $ 
\STATE \hspace{0.1cm} {\bf precondition}:
\STATE \hspace{0.2cm}  $messageLost >  \mathcal {K}  $  
\STATE \hspace{0.1cm} {\bf effect}:
\STATE \hspace{0.2cm}  $OM\gets 1$   

\STATE {\bf external} $GoToCooperative$ 
\STATE \hspace{0.1cm} {\bf precondition}:
\STATE \hspace{0.2cm} $OM =1 \vee  receive$ $ OM  =  0$ $from$ $all$ $other$ $Nodes$
\STATE \hspace{0.1cm} {\bf effect}:
 
\STATE \hspace{0.3cm} $OM \gets 0$

\STATE {\bf internal} $Progress$
\STATE \hspace{0.1cm} {\bf precondition}:
\STATE \hspace{0.2cm} $ actionCompleted  = True $
\STATE \hspace{0.1cm} {\bf effect}:
\STATE \label{ln:collected} \hspace{0.2cm} $ state \gets Wait  $ 
\STATE \label{ln:collected} \hspace{0.2cm} $ mbrInW  \gets \{Id\}  $ 

\STATE {\bf internal} $Wait$
\STATE \hspace{0.1cm} {\bf precondition}:
\STATE \hspace{0.2cm} $ mbrInP = mbrInW \vee |mbrInV| > 0$
\STATE \hspace{0.1cm} {\bf effect}:
\STATE \label{ln:collected} \hspace{0.2cm} $ state \gets Vote  $ 
\STATE \label{ln:collected} \hspace{0.2cm} $ mbrInV  \gets mbrInV \cup \{Id\}  $ 

\STATE {\bf internal} $Vote$
\STATE \hspace{0.1cm} {\bf precondition}:
\STATE \hspace{0.2cm} $ mbrInW = mbrInV$
\STATE \hspace{0.1cm} {\bf effect}:
\STATE \label{ln:collected} \hspace{0.2cm} $ state \gets Progress  $ 
\STATE \label{ln:collected} \hspace{0.2cm} $ mbrInV  \gets \emptyset$ 
\STATE \label{ln:collected} \hspace{0.2cm} $ mbrInP  \gets \{id\}$ 
 
\end{algorithmic}
\end{multicols}
\end{algorithm}

\subsection{Implementation}

We implement the described Kalman filter using the real-time publish-subscribe system. Thus, we create dedicated tasks for each sensor.  Namely, \emph{IMU} that consists of reading the gyroscope/accelerometer data, 
\emph{optical flow} that reads the offset between two consecutive images and \emph{ultra-wideband} that reads the ranging between the sensor and the fix stations. We register these tasks as publishers and publish the data at the updated refresh 
time. The core of the Kalman filter is implemented using two tasks: 

\begin{enumerate}
\item The \emph{state} task. The \emph{state} subscribes to \emph{IMU}, \emph{optical flow} and \emph{ultra-wideband}. When a publisher publishes new data, the state updates $Z$  and sets the uncertainties of the sensor in $R$, accordingly.

\begin{enumerate}
\item MPU 9250: The MPU returns the linear acceleration as a unitary vector. To convert to  $m/s^2$, we  multiply by $G=9.81 m/s^2$.  The MPU also returns the angular speed in degree per seconds. Therefore, we simple convert to radians. 

\item Flow Breakout (Optical flow): It returns the difference of pixels in $x$ and $y$ from two consecutive images. 
\begin{figure}[ht!]
\centering
\includegraphics[scale=0.6]{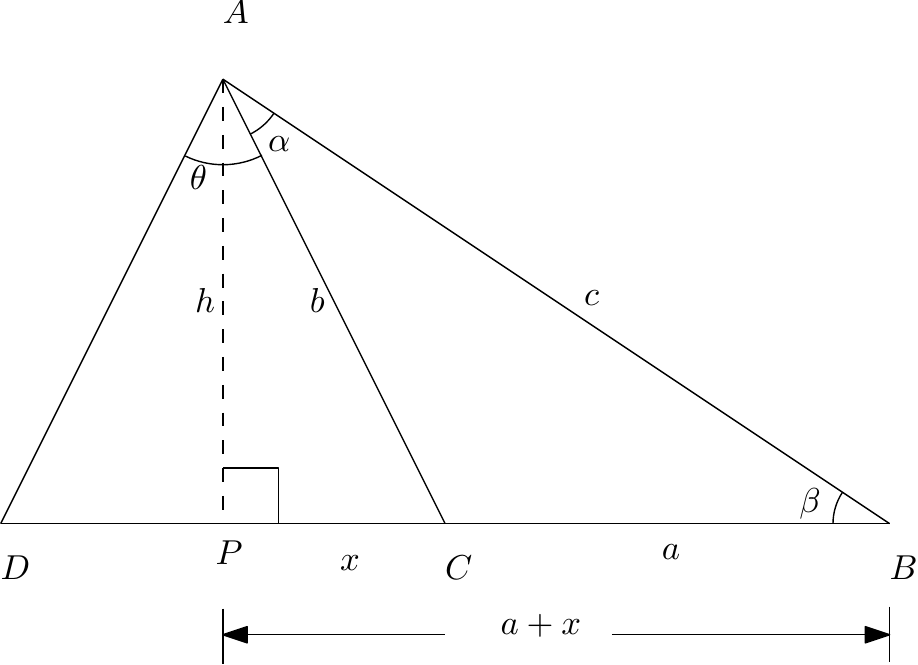}
\caption{Flow Breakout data with orientation $\alpha$.}
\label{fig:FlowBreakout}
\end{figure}
As depicted in the Figure~\ref{fig:FlowBreakout}, the Flow Breakout sensor measures the lateral velocity of the sensor in both directions, that is $\delta_x$ and $\delta_y$. The effective viewing angle of the sensor $\theta$ is assumed to be equal to 25$^{\circ}$. Since there will be a slight orientation in the sensor's axis we are adding an angle $\alpha$ to be the angle of orientation. Let $h$ be the height of the sensor from the ground. Observe that $\beta = \dfrac{\pi - \theta}{2} - \alpha $ and $a+x$ can be calculated using $\beta$ as  we already know the height of the sensor. That is, $a+x = h \cot (\beta)$. The Flow Breakout sensor measures the lateral velocity in pixels per second. The number of pixels in the $x$ and $y$ axis of the image frame is $30$ each. Therefore,  the
normalized speed is $x =\frac{h \cot (\beta) - a}{30}$.

\item DWM1000 (Ultra-wide band): The sensor returns the distance to a fixed number of anchors. An anchor is a station at a well known position. We compute the position using simple trigonometry.

We consider four anchors with a layout depicted in Figure~\ref{fig:counter} where anchor $a_0, a_1, a_2, a_3$ are located at  $(0,0,0)$, $(1,1,0)$, $(0,1,1)$ and $(1,0,1)$, respectively. Let $p = (x, y, z)$ be the robot position and let $[p_1, p_2, p_3]$ denote the  plane determined by the three points $p_1, p_2, p_3$. 

Consider any two different anchors $a_i,a_j$, the line segment $(a_i, a_j)$ and its  perpendicular line segment $(q,p)$ such that $q$ is on $(a_i, a_j)$. Observe that the perpendicular line has an angle of $45^o$ and $dist(a_i, a_j) = \sqrt{2}$. Let $dist(a_i, q) = b$, $dist(q, a_j) = \sqrt{2} -b$, $dist(q, p) = c$, $dist(a_i, p) = d_i$ and $dist(a_j, p) = d_j$ where $d_i$ and $d_j$ are returned by DWM1000. From the Pythagorean theorem, $b^2 + c^2 = d_i^2$ and $(\sqrt{2}-b)^2 + c^2 = d_j^2$. Therefore, $b = \frac{d_i^2-d_j^2 + 2}{2\sqrt{2}}.$

\begin{figure}[h!]
\centering
\includegraphics[scale=0.5]{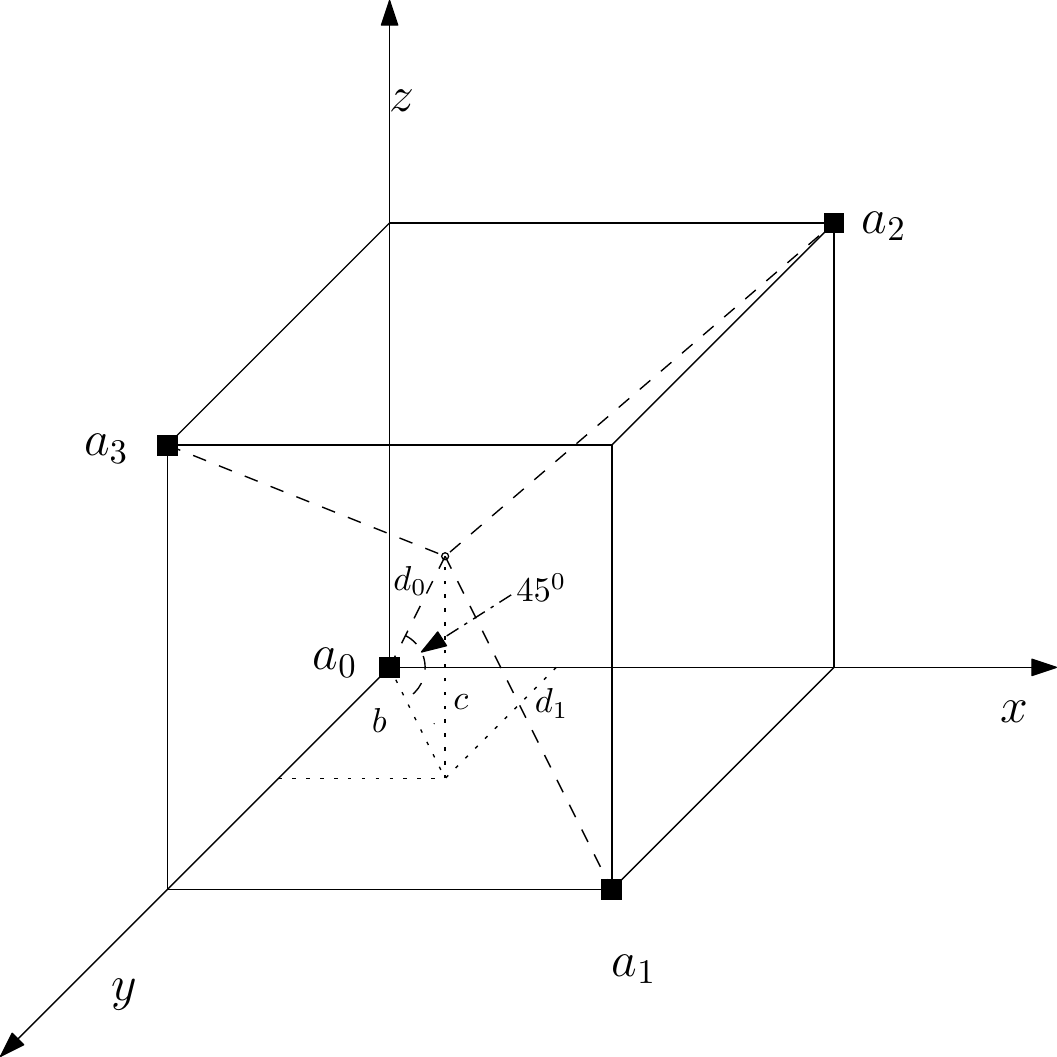}
\\ ~\\
\includegraphics[scale=0.6]{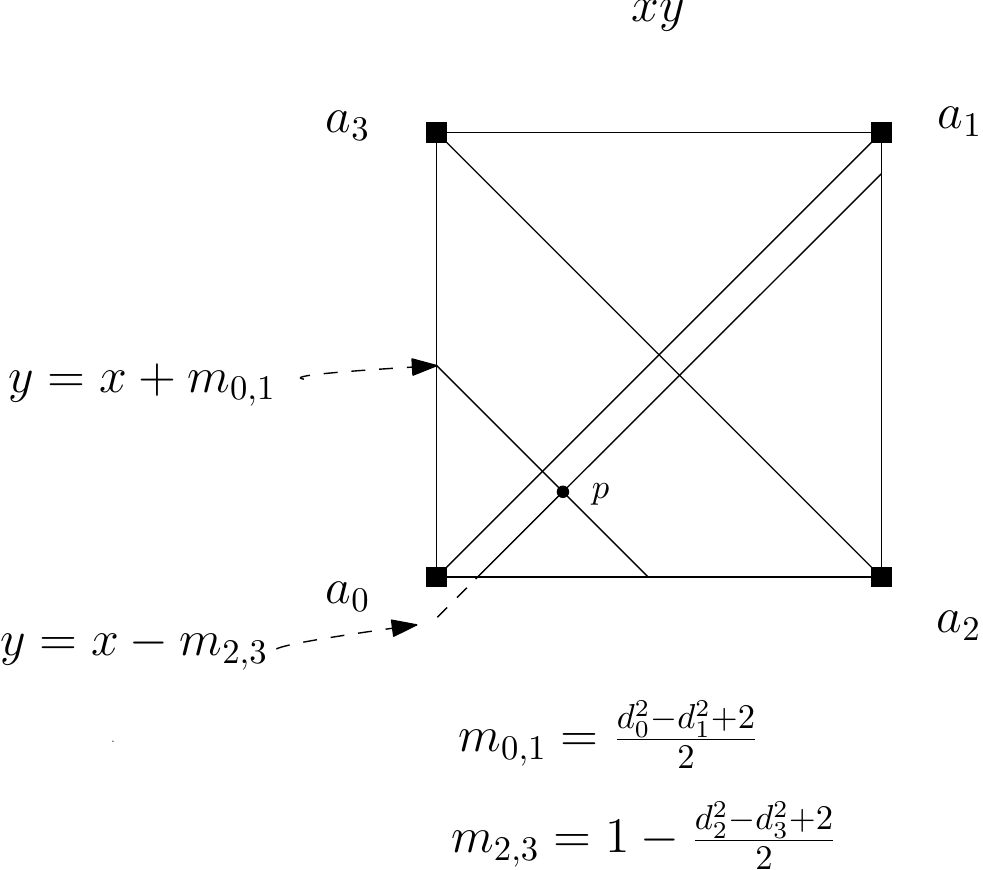}

\caption{Upper: Layout, Lower: Plane $xy$.}
\label{fig:counter}
\end{figure}

Observe that $q = (x, y)$ is the intersection of the perpendicular lines of $(a_0,a_1)$, $(a_2,a_3)$ in the plane $XY$, $q = (x, z)$ is the intersection of the perpendicular lines of $(a_0,a_2)$, $(a_1,a_3)$ in the plane $XZ$ and  $q = (y, z)$ is the intersection of the 
perpendicular lines of $(a_0,a_3)$, $(a_1,a_2)$ in the plane $YZ$.

Let $y = x - m_{0,1}$ be the equation of the perpendicular line $(a_0,a_1)$ passing through $p$. Since the angle is $45^o$, $m_{0,1} = \frac{d_0^2-d_1^2 + 2}{2}$; see Figure~\ref{fig:counter}.  Similarly, let $y = -x + m_{2,3}$ be the equation of the perpendicular $(a_2,a_3)$ passing through $p$. Therefore, $m_{2,3} = 1- \frac{d_2^2-d_3^2 + 2}{2}$. Similarly we can compute for the other perpendicular; see Figure~\ref{fig:counter}.


We can now compute the position using the following equations:
\[
\begin{array}{rcl}
x &=& \frac{m_{0,1}+m_{2,3}}{2} = \frac{m_{0,2} + m_{1,3}}{2} \\
y &=& \frac{m_{0,1}-m_{2,3}}{2} = \frac{m_{0,3} + m_{1,2}}{2} \\
z &=& \frac{m_{0,2}-m_{1,3}}{2} = \frac{m_{0,3} + m_{1,2}}{2} 
\end{array}
\]

\end{enumerate}

\item  The Kalman task executes at a fixed rate. First it updates the state using  equation~\ref{eqn:model}~and~\ref{eqn:correction} and then  it sets  high uncertainties in $R$ for the new calculations. Recall that in the Kalman filter the weight of each sensor information depends on the uncertainties. Once a new sensory reading is completed the uncertainties matrix $R$ is adjusted. Observe that if one step of the Kalman filter is executing without reading the sensory data, the state will be determined by the model. This approach allows the Kalman filter to handle sensors reading at different rates and reduces the maintenance time.
\end{enumerate}

We show a sketch of the code that receives the data to populate the matrices $Z$ and $R$.
\begin{scriptsize}
\begin{verbatim}
void stateTask(void* args)
  while (true)
    if (xTaskNotifyWait(0xffffffff, 0xffffffff, &publisherId, 
             portMAX_DELAY) == pdTRUE) 
    switch (publisherId)
     case MPU_DATA:
       xQueueReceive(mpuQueue, &gyroAcc, 
          (TickType_t )0);   // gyroAcc contains the data
                                         of the IMU
    case FLOW_DATA:
      xQueueReceive(flowQueue, &flowData, 
           (TickType_t )0);   // flowData contains the data 
                                         of the flowdeck  
    case DWM_DATA:
      xQueueReceive(dwmQueue, &dwmData, 
            (TickType_t )0);  // dwmData contains the data
                                         of the DWM100 
\end{verbatim}
\end{scriptsize}

The task of the Kalman Filter is executed at the frequency of $FREQUENCY\_KF$. It  executes one step of the Kalman filter, update the variables for the next step and publish the new state. After publishing the data, the matrices $Z$ and $R$ are reset to high uncertainties. 
\begin{scriptsize}
\begin{verbatim}
void kalmanTask(void* args)
     while (true)
       vTaskDelayUntil(&lastWakeTime, 
                     FREQUENCY_KF / portTICK_RATE_MS);    
       if  (ekf_step(&ekf))  // Executes the Kalman Filter
           xQueueOverwrite(stateQueue, &state);
                                  // insert in the queue
           publish(stateId);  // publish
       // reset the Matrices Z and R	
\end{verbatim}
\end{scriptsize}

\section{Membership and Synchronizer Service}\label{sec:ync}
In this section, we describe the protocol which allows the synchronization of Timed Input/Output Automata. We present the Algorithm as a Timed Input/Output Automata in  Protocol~\ref{alg:actions2}.  Let $R= \{r_1,r_2,..., r_n\}$ denote the set of robots. The protocol relies on three states: $PROGRESS$ where robots are performing a maneuver, $WAIT$ where robots have completed the maneuver and $VOTING$ where robots are voting to start the next instruction. Let $M_i$, initially set to $\emptyset$, be the set of boolean values that indicates when a robot $r_i$ has completed the maneuvers, and let $v_i$ denote the number of rounds in voting of robot $r_i$

For simplicity of presentation, we assume that every robot is aware of the rest of the robots. Later we explain how to remove this assumption. Every robot broadcasts its id, state of the automata and  the number of rounds it has been in the $VOTING$ state. While robots are in the state  $PROGRESS$, they  are performing a maneuver.  When a robot completes it maneuver, it changes its state to $WAIT$ and waits for the other robots to complete their maneuver. When a robot $r_i$ receives the $WAIT$ state from robot $r_k$, it sets a $M_i[k]=true$. Robot $r_i$ changes its state from $WAITING$ to $VOTING$ when $\wedge_{i=1}^n M_i = TRUE$ or  if it receives the $VOTING$ state from a robot $r_k$ in which case it sets $v_i = v_k$. Robot $r_i$ increases $v_i$ and when $v_i = \mathcal{K}$, it changes to $PROGRESS$, sets $M_i = \emptyset$, $v_i=0$ and start a new maneuver.

In the following theorem we show that if the number of consecutive rounds with message lost in the system is less than $\mathcal{K}$, 
all robots start a new maneuver at the same time.  We assume that all the maneuver take a bounded number of  rounds and at least $\mathcal {K}$ rounds.

\begin{theorem}~\label{ref:bestcase}
If the number of consecutive rounds  with message lost in the system is less than $\mathcal{K}$, Algorithm~\ref{alg:actions2} satisfies the following properties: \\
1) Safety: If a robot is in $WAIT$ state, then eventually all robots reach $WAIT$ state. Further, if a robot changes its state to $IN\_PROGRESS$  at round $r$, all other robots change the state to $IN\_PROGRESS$  at round $r$, and\\
2) Liveness:  Each robot executes infinitely often in $IN\_PROGRESS$ state.
\end{theorem}

\begin{proof}
We first prove that property (1) holds under the assumptions. Consider any robot $r_i$ that reaches the $WAIT$ state. Observe that, $r_i$ remains in $WAIT$ until it receives that all other robots are in the $WAIT$ state  since $\wedge_{i=1}^n M_i = FALSE$ or it receives  $VOTING$ from a robot. Therefore, the earliest time that a robot can change its state to $VOTING$ is when all robots are in $WAIT$ for at least one round.  Let $m_1$ be the round where the last robot changes its state to $WAIT$.  Observe that there is a round in the interval $[m_1, m_1 + \mathcal{K}]$ where all robots receive the message with the state $WAIT$, since the number of consecutive rounds  with message lost in the system is less than $\mathcal{K}$. Consider the first robot $r_i$ that listens the $WAIT$ from all other robots in round $m_2 \in [m_1, m_1 + \mathcal{K}]$. Therefore, $r_i$ changes its state to $VOTE$ and broadcasts $v_i + j$ in each round $j \in [m_2 + m2+\mathcal{K}]$.  Observe that $r_i$ changes its state to $PROGRESS$ when $v_i  + j > \mathcal{K}$. When any other robot, say $r_k$, receives the message from  $r_i$ at any round during the interval $[m_2, m_2 + \mathcal{K}]$ it changes its state to $VOTING$ and set $v_k = v_i +j$. Therefore, every other robot changes its state to $PROGRESS$ when $v_i  + j > \mathcal{K}$.

It is not difficult to see that Property (2) holds under the assumptions that the number of consecutive rounds with message loss  is less than $\mathcal{K}$.
\end{proof}

Observe that from Theorem~\ref{ref:bestcase}, robots reach a consensus to start a new maneuver at the same time. Therefore, it provides  a synchronous robot system. Although one can expect that the number of consecutive rounds with message loss is less than $\mathcal{K}$, in practice we can have communication faults, i.e., the number of consecutive rounds with message loss is at least $\mathcal{K}$.  Therefore, properties (1) and (2) may not hold. Indeed, it is known that the consensus problem cannot be solved deterministically if messages can be lost. To alleviate the problem, we consider two operation modes. Namely $COOPERATIVE$ mode where every robot receives the message from every other robot and $AUTONOMOUS$  mode where at least one robot does not receive the message from other robot.  Observe that if there is a partition among the robots such that some robots are in the $AUTONOMOUS$ mode and others are in $COOPERATIVE$ mode, a cooperative robotic application may not be completed successfully or robots may collide. Therefore, the goal is to let robots agree on the operation mode. However, it is again a consensus problem and consequently there is no deterministic algorithm that can solve it. However, we can minimize the number of rounds where robots do not agree using the protocol proposed in~\cite{morales2018stop}.  In this protocol, when a robot does not  receive a message at round $r$ from another robot or when it receives $AUTONOMOUS$ operation mode from any other robot, then it changes its operation mode to $AUTONOMOUS$. The robots return to $COOPERATIVE$ operation mode when they receive the $AUTONOMOUS$ operation mode from every other robot. 

\begin{lemma}
\label{lemma2}
{\bf Stream Consensus~\cite{morales2018stop}}
Robots may disagree in at most two continuous  rounds.
\end{lemma}

Consider the case where robots are in $WAIT$ or $VOTING$ state. Observe that if the robots change  to $AUTONOMOUS$ operation mode due to message loss, then it is safe to start the next maneuver when they change to $COOPERATIVE$ operation mode. From Lemma~\ref{lemma2}, robots may start within two rounds of difference. However, when robots are in the $PROGRESS$ state, they always send $COOPERATIVE$ as the operation mode and do not execute the \emph{Stream Consensus} protocol. The following theorem summarizes the the previous discussion and is presented without a proof.

\begin{theorem}~\label{theorem2}
Protocol~\ref{alg:actions2} allows the robots to start a new maneuver within at most two rounds of difference.  Further, all robots have common knowledge before starting  a new maneuver.
\end{theorem}

Now, we show how to remove the assumption that robots need to know the number of active robots. Indeed, during the $PROGRESS$ state, robots implement a membership service. More precisely, robots maintain the set $members_i$ of participating robots that they receive messages from. Further, robots maintain the set  $membersInWait_i$ of robots that are in the $WAIT$ state. Robots broadcast $membersInWait_k$ while they are in the $WAIT$ state. A robot decides to change its state to $VOTING$ when it receives $membersInWait_k = members_i$ from every other robot.

\section{Virtual Machine}\label{sec:language}
In this section we present the implementation of the programming interface that is used to implement robotic applications in the framework which consists of three callback functions and one method to send  messages; see Interface~\ref{alg:signature2}.

\floatname{algorithm}{Interface}
\begin{algorithm}[ht!]
\begin{algorithmic}[1]
\caption{Programming Interface}
\label{alg:signature2}
\STATE \void *init()
\STATE  \byte (*newRound)(CarCommand *command, LDMap *ldmap, \byte globalState, \inte clock)
\STATE  \byte (*endOfRound)(LDMap *ldmap, \byte globalState, \inte clock)
\STATE  \void sendMessageData(\byte  *payload) 
\end{algorithmic}
\end{algorithm}

\begin{itemize}
\item \emph{init.} It is called  when the system starts. It is used to initialize the local state of the robot. 
\item \emph{newRound.} It is called at the beginning of each application round. In other words, when the robots change their state to $PROGRESS$. 
This function allows  setting the robot's commands.  
Application rounds are suggested to be of 1 or 2 seconds to let the robot make progress.
The function receives the local dynamic map that contains the state of all the participant robots at the end of previous round.
If the operation mode in the local dynamic map is $COOPERATIVE$, then all robots agree in the content of the local dynamic map.
It also receives the globalState and the local clock. 
\item \emph{endOfRound.}  The function receives the local dynamic map that contains the state of all the participant robots in the end of the round, i.e., 
when the robot complete the maneouver. It also receives the globalState, and the local clock. The TIOA is responsible for performing the transitions of globalState 
accordingly.%
\item {sendMessageData.} It allows to broadcast up to 9 bytes in each round. 
\end{itemize}

We enhance  \sysname~ by implementing a set of  local and distributed queries to simplify the design of cooperative robotic applications:
\begin{itemize}
\item \emph{NumberOfMembers.}  Returns the number of active members performing the current primitive.
\item \emph{Leader.} Returns the leader defined as the robot with minimum id. 
\item \emph{ShortestEnclosingCircle.} Returns the center and the radius of the shortest circle that contains all the robots. 
\item \emph{ConvexHull.} Returns the list  of active robots in the convex hull. 
\item \emph{MinWeightedMatching.} Given a set of points, it returns the matching point for each robot that minimizes 
the maximum sum of distances from each point to its destination. 
\end{itemize}

\section{Conclusion}\label{sec:conclusion}

In this paper, we present the current work on  \sysname, a testbed that allows undergraduate students to implement and demo collaborative robotic applications. \sysname~ provides a synchronous robotic system to implement cooperative applications. We expect that the testbed lowers the barriers for the students of computer science and computer engineering at CSULB. 
We plan to perform extensive experiments with different robots. The  implementation has demonstrated that \sysname~is feasible.

\end{document}